\documentclass[conference]{IEEEtran}
\IEEEoverridecommandlockouts
\usepackage{cite}
\usepackage{amsmath,amssymb,amsfonts,amsthm}
\usepackage{graphicx}
\usepackage{textcomp}
\usepackage{xcolor}

\usepackage{bm}
\usepackage{mathtools}
\usepackage{float}
\usepackage[ruled,vlined,linesnumbered]{algorithm2e}
\usepackage{algpseudocode}
\usepackage{leftidx}
\usepackage{url}
\usepackage{tabularx, multirow, multicol}
\usepackage{color}
\usepackage{soul}
\usepackage{caption}
\usepackage{subcaption}
\usepackage{float}
\usepackage{balance}
\usepackage{setspace}

\newcommand*{\myov}[1]{\overbracket[0.45pt][-1pt]{#1}}

\newenvironment{talign*}
 {\let\displaystyle\textstyle\csname align*\endcsname}
 {\endalign}

\newcommand{\gsum}{\textsc{sum}\xspace}
\newcommand{\gmax}{\textsc{max}\xspace}
\newcommand{\gmean}{\textsc{mean}\xspace}
\newcommand{\gset}{\textsc{set2set}\xspace}
\newcommand{\gattn}{\textsc{attention}\xspace}
\newcommand{\gsort}{\textsc{sort}\xspace}

\newcommand{\proposed}{SSRead\xspace}
\newcommand{\gread}{GRead\xspace}
\newcommand{\graphcl}{GraphCL\xspace}
\newcommand{\graphim}{InfoGraph\xspace}

\newcommand{\gcn}{GCN\xspace}

\newcommand{\gin}{GIN\xspace}
\newcommand{\dgcnn}{DGCNN\xspace}
\newcommand{\sagpool}{SAGPool\xspace}
\newcommand{\diffpool}{DiffPool\xspace}
\newcommand{\gunet}{GUNet\xspace}

\newcommand{\sortpool}{SortPool\xspace}
\newcommand{\eigenpool}{EigenPool\xspace}
\newcommand{\topkpool}{TopKPool\xspace}

\newcommand{\dnd}{D\&D\xspace}
\newcommand{\proteins}{PROTEINS\xspace}
\newcommand{\nci}{NCI1\xspace}
\newcommand{\mutag}{MUTAG\xspace}
\newcommand{\mutagen}{Mutagen\xspace}
\newcommand{\imdbb}{IMDB-B\xspace}
\newcommand{\imdbm}{IMDB-M\xspace}

\newcommand{\uset}{\mathcal{D}}

\newcommand{\mmat}{\bm{M}}
\newcommand{\amat}{\bm{A}}
\newcommand{\xmat}{\bm{X}}
\newcommand{\hmat}{\bm{H}}
\newcommand{\gmat}{{\myov{\bm{H}}}}
\newcommand{\pmat}{\bm{P}}
\newcommand{\wmat}{\bm{W}}

\newcommand{\svec}{\bm{s}}
\newcommand{\zvec}[1]{\bm{z}_{#1}}

\newcommand{\hvec}[1]{\bm{h}_{#1}}
\newcommand{\gvec}[1]{{\myov{\bm{h}}}_{#1}}
\newcommand{\pvec}[1]{\bm{p}_{#1}}
\newcommand{\wvec}[2]{\bm{w}_{#2,#1}}

\newcommand{\softssa}{\text{SSA}_{\gamma}}

\newcommand{\closs}{\mathcal{L}_{class}}
\newcommand{\sloss}{\mathcal{L}_{contr}}
\newcommand{\ploss}{\mathcal{L}_{align}}

\newcommand{\smallsection}[1]{{\vspace{0.05in} \noindent \bf {#1.\hspace{5pt}}}}
\DeclareMathOperator*{\expect}{\mathbb{E}}
\DeclareMathOperator*{\argmin}{argmin}
\newtheorem{theorem}{Theorem}

\def\BibTeX{{\rm B\kern-.05em{\sc i\kern-.025em b}\kern-.08em
    T\kern-.1667em\lower.7ex\hbox{E}\kern-.125emX}}
\begin{document}

\title{Learnable Structural Semantic Readout for Graph~Classification}

\author{
\IEEEauthorblockN{Dongha Lee$^1$, Su Kim$^2$, Seonghyeon Lee$^2$, Chanyoung Park$^3$, Hwanjo Yu$^{2*}$}
\IEEEauthorblockA{
$^1$\textit{University of Illinois at Urbana-Champaign (UIUC), Urbana, IL, United States}\\
$^2$\textit{Pohang University of Science and Technology (POSTECH), Pohang, Republic of Korea}\\
$^3$\textit{Korea Advanced Institute of Science and Technology (KAIST), Daejeon, Republic of Korea}\\
donghal@illinois.edu, \{kimsu, sh0416, hwanjoyu\}@postech.ac.kr, cy.park@kaist.ac.kr}
\thanks{* Corresponding author}
}

\setlength{\abovecaptionskip}{5pt}
\setlength{\belowcaptionskip}{5pt}
\setlength{\floatsep}{7pt plus 1.0pt minus 2.0pt}
\setlength{\textfloatsep}{5pt plus 1.0pt minus 2.0pt}
\setlength{\dblfloatsep}{10pt plus 1.0pt minus 2.0pt}
\setlength{\dbltextfloatsep}{7pt plus 1.0pt minus 2.0pt}

\newcolumntype{L}{>{\raggedright\arraybackslash}m{0.31\linewidth}}
\newcolumntype{Z}{>{\raggedright\arraybackslash}m{0.2\linewidth}}
\newcolumntype{P}{>{\centering\arraybackslash}m{0.05\linewidth}}
\newcolumntype{Q}{>{\centering\arraybackslash}m{0.09\linewidth}}
\newcolumntype{R}{>{\raggedleft\arraybackslash}m{0.095\linewidth}}
\newcolumntype{S}{>{\centering\arraybackslash}m{0.1\linewidth}}

\let\oldnl\nl
\newcommand{\nonl}{\renewcommand{\nl}{\let\nl\oldnl}}

\maketitle

\begin{abstract}
With the great success of deep learning in various domains, graph neural networks (GNNs) also become a dominant approach to graph classification.
By the help of a global readout operation that simply aggregates all node (or node-cluster) representations, 
existing GNN classifiers obtain a graph-level representation of an input graph and predict its class label using the representation.
However, such global aggregation does not consider the structural information of each node, which results in information loss on the global structure.
Particularly, it limits the discrimination power by enforcing the same weight parameters of the classifier for all the node representations;
in practice, each of them contributes to target classes differently depending on its structural semantic.
In this work, we propose structural semantic readout (\proposed) to summarize the node representations at the position-level, which allows to model the position-specific weight parameters for classification as well as to effectively capture the graph semantic relevant to the global structure.
Given an input graph, \proposed aims to identify structurally-meaningful positions by using the semantic alignment between its nodes and structural prototypes, which encode the prototypical features of each position.
The structural prototypes are optimized to minimize the alignment cost for all training graphs, while the other GNN parameters are trained to predict the class labels.
Our experimental results demonstrate that \proposed significantly improves the classification performance and interpretability of GNN classifiers while being compatible with a variety of aggregation functions, GNN architectures, and learning frameworks.
\end{abstract}

\begin{IEEEkeywords}
Graph classification, Graph neural networks, Global structural information, Learnable graph readout
\end{IEEEkeywords}

\section{Introduction}
\label{sec:intro}
Graph classification refers to the task of predicting class labels of input graphs, and it has been applied to a wide range of graphs, including molecular structures~\cite{wale2008comparison}, biological networks~\cite{borgwardt2005protein}, and social networks~\cite{yanardag2015deep}.
The key challenge is to extract informative (or discriminative) graph features from topological structures (i.e., nodes and edges) and auxiliary node features.
Early work on graph classification utilized a variety of graph kernels~\cite{borgwardt2005shortest, shervashidze2011weisfeiler} to generate task-agnostic features from complex graph structures and use them for classification.
However, such kernel-based approaches incur expensive computational cost which is quadratic (or sometimes cubic) with respect to the number of training graphs and their nodes~\cite{wu2020comprehensive}.
In addition, they are not able to jointly learn graph features with an off-the-shelf classifier such as support vector machine (SVM);
this makes the features non-adaptive to a target task~\cite{errica2019fair}.

Recently, there have been a lot of attempts to exploit graph neural networks (GNNs) for graph classification~\cite{ying2018hierarchical, xu2018powerful, lee2019self, ma2019graph, errica2019fair}, and they have shown promising results without expensive computations required by graph kernels.
They aim to learn the effective representations (i.e., low-dimensional latent vectors) of input graphs while being trained to predict the class labels based on the representations.
Most of them basically adopt the message passing architecture such as graph convolutional networks (GCNs)~\cite{kipf2017semi} that compute each node representation based on its local graph structure.
Through hierarchical graph pooling~\cite{ma2019graph, lee2019self, ying2018hierarchical, bianchi2020spectral} and readout~\cite{li2015gated,vinyals2015order,xu2018powerful} which summarize the node representations, they finally obtain a graph-level representation of the entire input graph.

\begin{figure}[t]
	\centering
	\includegraphics[width=\linewidth]{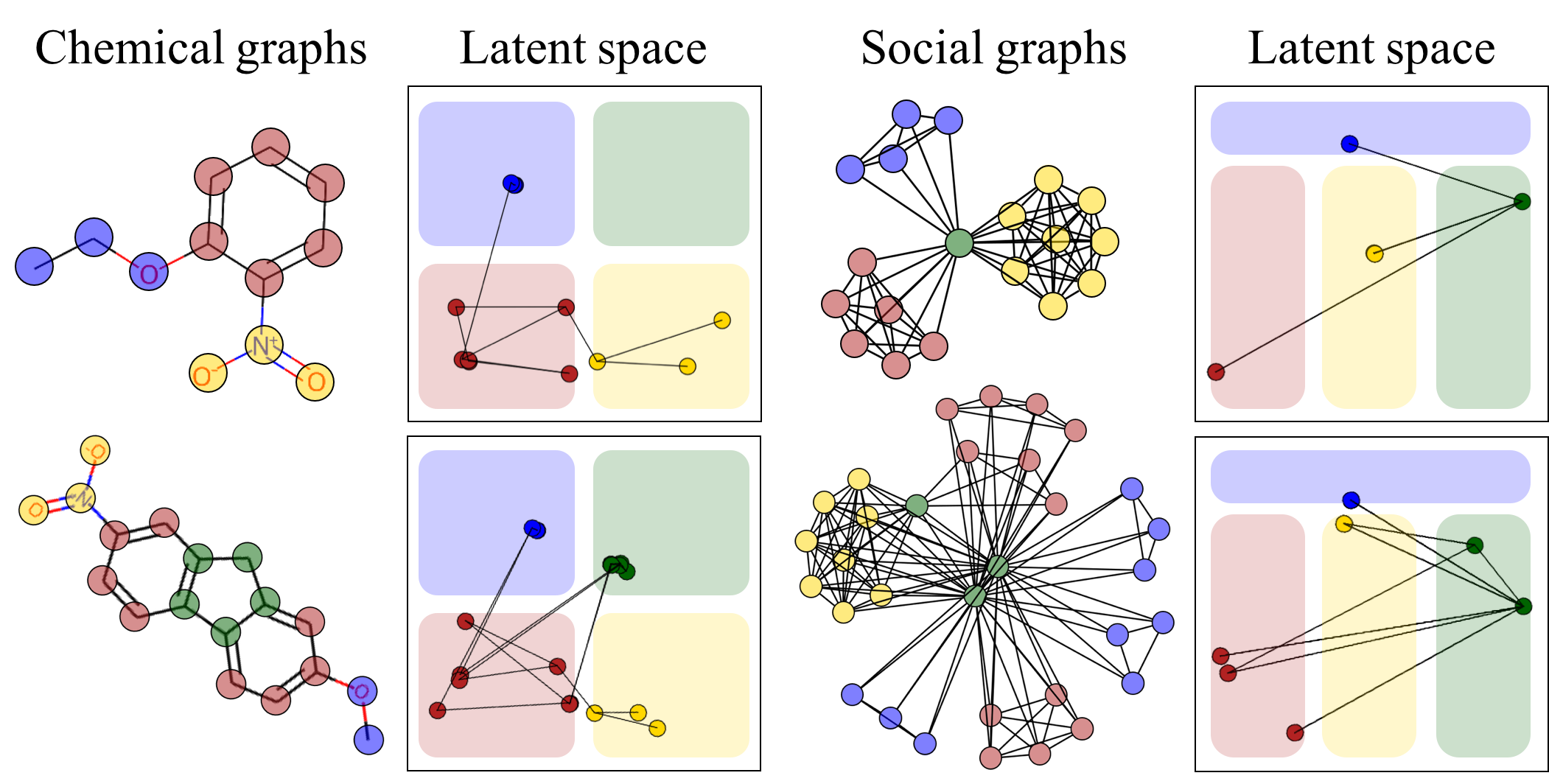}
	\caption{Examples of graph-level structural information. Two input graphs are consistently morphed in the space where structural semantic is encoded (plotted by t-SNE), and the structural positions are marked in four different colors.}
	\label{fig:example}
\end{figure}

A recent challenge for improving the expressivity of the GNNs is leveraging the global (or graph-level) structural information. 
In case of convolutional neural networks (CNNs) for encoding images, 
they fully utilize the fixed position of each pixel by the help of spatial convolution and spatial pooling, thereby obtaining a structured feature map that preserves the spatial information.
Unlike images whose pixels have consistent positions in a structured way, graphs implicitly represent the global position of their nodes by an adjacency matrix whose node-order can be permuted (i.e., graph isomorphism).
In this sense, most studies have developed GNN modules that capture such global position information into the representations at the node-level~\cite{bouritsas2020improving, pei2020geom, you2019position} (e.g., structure-aware message passing which further considers structural information of neighbor nodes), or at the subgraph-level~\cite{ma2019graph, bianchi2020spectral, yuan2020structpool} (e.g., graph pooling based on spectral clustering).

Despite these efforts, the global structural information has not been carefully considered at the graph-level representations yet.
Note that every GNN classifier uses a global readout operation, which simply aggregates all remaining node (or subgraph) representations, to obtain a permutation-invariant graph-level representation~\cite{ying2018hierarchical, xu2018powerful, lee2019self, gao2019graph}. 
In this work, we point out that the global readout does not consider the structural information of each node, which incurs information loss on the global structure of an input graph.
As depicted in Figure~\ref{fig:example}, graphs are composed of several structurally-meaningful positions, termed as \textit{structural positions}; e.g., different functional groups for molecular graphs, and communities with different densities for social graphs. 
In this case, a globally-aggregated representation cannot accurately express the inclusive graph semantic relevant to the structural positions.
Moreover, it limits the capability of the GNN classifier in that the classification layer considers all the aggregated nodes by using the global weight parameters.
That is, it cannot model different interactions of each node (atom) with the target class (molecular property) depending on its structural position (functional group).

To tackle this challenge, we propose a novel graph readout technique, \underline{S}tructural \underline{S}emantic \underline{Read}out (\proposed), that outputs the graph-level representation explicitly keeping the global structural information.
Motivated by consistently-morphed graphs in the latent space according to its structural semantic (Figure~\ref{fig:example}), our \proposed takes advantage of consistent positions in the latent space, which eventually correspond to the structurally-meaningful positions in each graph.
We first introduce trainable parameters for representing the predefined number of the structural positions, referred to as \textit{structural prototypes}, to encode the latent semantic (i.e., prototypical features) of the positions seen in the training graphs.
Using the structural prototypes, \proposed aligns each node representation with its semantically-closest position, then it aggregates the set of node representations aligned with each position.
By doing so, \proposed produces a permutation-invariant graph-level representation that consists of multiple-and-consistent positions.
In the end, it allows to model the classification layer by using position-specific weight parameters.

As a part of the GNN classifier trained for its target task, the structural prototypes are simultaneously optimized so that the total alignment cost is minimized for all training graphs.
However, the optimal cost for the hard (i.e., discontinuous) alignment between the nodes and the positions is not differentiable.
Thus, in order to effectively update the parameters by using its gradient, we utilize the soft-relaxation of the node-position alignment cost which considers all possible alignments.
That is, the structural prototypes are optimized to be better aligned with the node representations, while the GNN parameters are optimized to more accurately classify an input graph.
These two types of optimizations mutually enhance each other in a unified manner.

Our extensive evaluation on graph classification benchmarks demonstrates the effectiveness of \proposed in terms of both classification performance and interpretability.
\proposed significantly enhances the discrimination power of a classifier for all the datasets, regardless of aggregation functions, GNN architectures, and learning objectives.
In particular, it achieves higher accuracy for predicting the structural properties (e.g., the number of rings in an input molecule) than the global readout, by effectively capturing the global structural information into its final graph-level representation.
Furthermore, our qualitative analyses on node-position alignment and its localization performance support that \proposed provides better interpretability of GNNs as well.

The main contributions are summarized as follows.
\begin{itemize}
    \item \textbf{Compatibility} --- \proposed can be easily embedded into any GNN architectures (i.e., compatible with a variety of message passing and graph pooling layers) and make use of any aggregation functions for its position-level readout.
    
    \item \textbf{Performance} --- \proposed improves the classification accuracy by learning the position-level graph representations with the position-aware classification layer, which exploits the global structural information. 
    
    \item \textbf{Interpretability} --- With the structural prototypes optimized from training data, \proposed can perform segmentation on a graph according to the structural semantic and localize further discriminative regions for the target class.
\end{itemize}

\section{Preliminary}
\label{sec:preliminary}
\subsection{Problem Formulation}
Let $\uset$ be a training set of graph instances ($G\in\mathcal{G}$) with their class labels ($y\in\mathcal{Y}$), where there exist $C$ classes, $\mathcal{Y}=\{1,\ldots,C\}$.
Each graph can be represented as an adjacency matrix $\amat\in\{0,1\}^{N\times N}$ and its node feature matrix $\xmat\in\mathbb{R}^{N\times F}$, where $N$ and $F$ are the numbers of nodes and node features, respectively.
The goal of our work is to train an effective \textit{graph encoder}
that maps input graphs into the low-dimensional latent representations, and also its \textit{classifier} 
that accurately predicts the class labels using the representations.

\subsection{Graph Neural Networks for Classification}
\label{subsec:gnn}
The most basic architecture of graph neural networks (GNNs) is a stack of graph convolutional layers~\cite{kipf2017semi} or graph attentional layers~\cite{velivckovic2018graph}.
Each layer of this architecture, also known as a \textit{message-passing} mechanism, computes the representation of each node by aggregating those of its neighbor nodes.
To be specific, starting from the input node features $\hmat^{(0)}=\xmat\in\mathbb{R}^{N \times F}$, the output of the $l$-th graph convolutional layer $\hmat^{(l)}\in\mathbb{R}^{N \times d}$ is computed by\footnote{For simplicity, we use the same dimensionality $d$ for all the layers.}
\begin{equation}
\label{eq:gcn}
    \hmat^{(l)} = \text{ReLU}\left(\tilde{\bm{D}}^{-\frac{1}{2}}\tilde{\amat} \tilde{\bm{D}}^{-\frac{1}{2}} \hmat^{(l-1)} \bm{W}^{(l)}\right),
\end{equation}
where $\tilde{\amat}=\amat+\bm{I}$ is the adjacency matrix with added self-loops, $\tilde{\bm{D}}$ is the degree matrix of $\tilde{\amat}$, and $\bm{W}^{(l)}$ is a trainable weight matrix.
The node representations computed by the last graph convolution (i.e., the $L$-th layer) is regarded as the output of a GNN node encoder, $\text{GNN}(G; \Theta)=\hmat^{(L)}\in\mathbb{R}^{N \times d}$ where $\Theta$ is the set of all trainable parameters in the layers.

To summarize the node representations into a higher-level graph representation, recent studies have developed various hierarchical pooling techniques that iteratively downsample the nodes while preserving the topological structure based on a bottom-up approach.
This type of graph pooling layer allows GNNs to attain scaled-down graphs in an end-to-end manner, based on score-based node selection~\cite{zhang2018end,lee2019self,gao2019graph} and node clustering~\cite{ying2018hierarchical,ma2019graph,bianchi2020spectral}.
In the end, a graph-level representation is obtained by a readout operation that aggregates the representations of all remaining nodes, then the final classification scores are computed based on the representation.
Figure~\ref{fig:gnn} illustrates the overall architecture of a GNN classifier.

\begin{figure}[t]
	\centering
	\includegraphics[width=\linewidth]{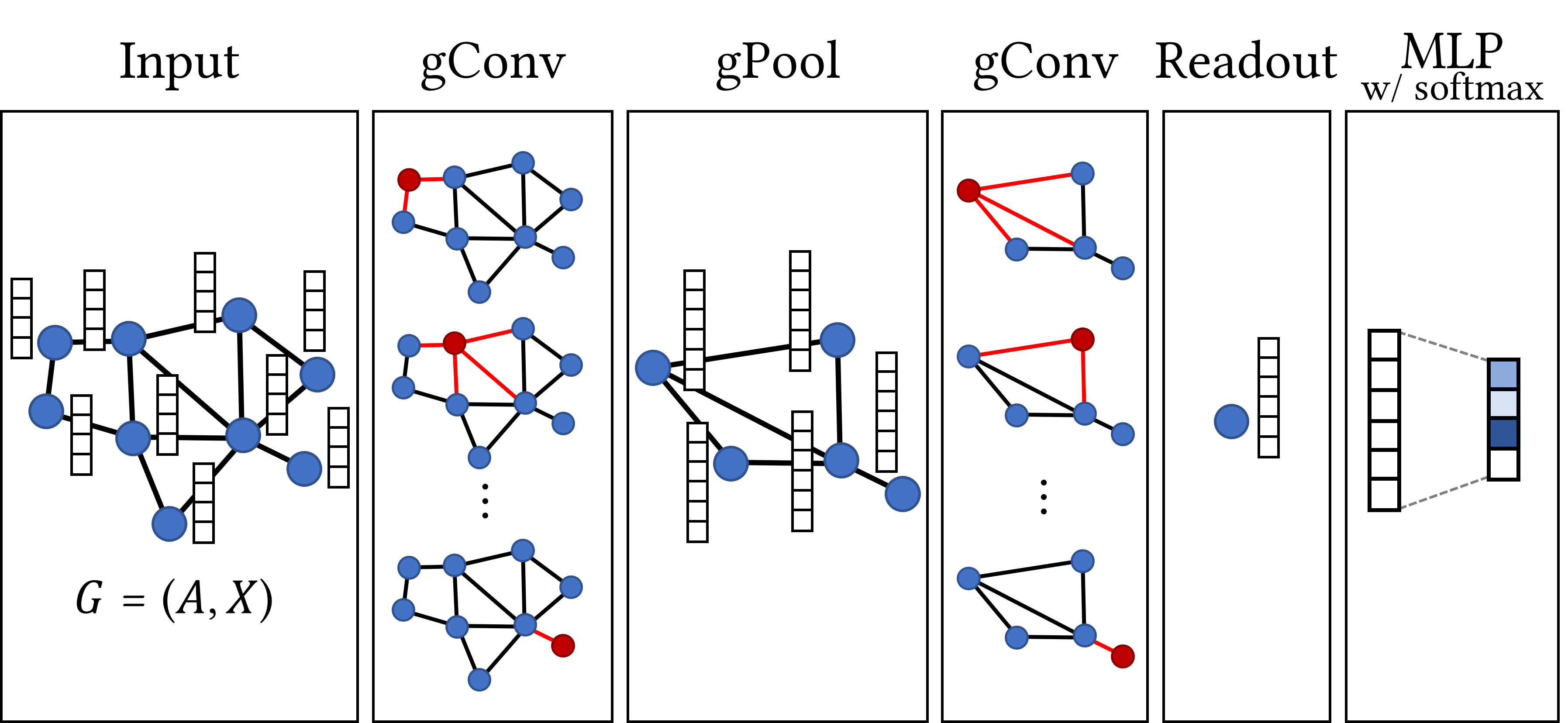}
	\caption{The overall architecture of a GNN classifier.}
	\label{fig:gnn}
\end{figure}

\section{Related Work}
\label{sec:related}


We review the literature on hierarchical graph pooling and clarify its difference from a readout operation.
Most of the existing graph pooling techniques that reduce the number of nodes from $N$ to $N'$ can be categorized as either node selection or node clustering.
First, \textit{node selection} methods select $N'$ nodes based on their self-attention scores~\cite{lee2019self} or projected scores~\cite{gao2019graph} with the other nodes discarded. 
That is, the pooled representation $\gmat\in\mathbb{R}^{N' \times d}$ and the corresponding adjacency matrix $\myov{\amat}\in\mathbb{R}^{N'\times N'}$ are obtained by using the indices of $N'$ selected nodes $idx =\text{select}(\text{score}(\hmat), N')$.
\begin{equation*}
    \gmat = \hmat[idx, :] \text{    and    } \myov{\amat} = \amat[idx, idx].
\end{equation*}
On the other hand, \textit{node clustering} methods reduce the size of a graph by pooling the node representations based on their assigned $N'$ clusters.
Given a cluster assignment matrix $C\in\mathbb{R}^{N \times N'}$ obtained by graph clustering, such as soft node-clustering~\cite{ying2018hierarchical, yuan2020structpool} and spectral clustering~\cite{ma2019graph, bianchi2020spectral}, the pooled representation is computed by
\begin{equation*}
    \gmat = \bm{C}^T\hmat \text{    and    }
    \myov{\amat} = \bm{C}^T\amat\bm{C}.
\end{equation*}
However, the hierarchical graph pooling is not capable of mapping input graphs into consistent representations.
To be precise, the pooled representation $\gmat$ is still row-wise permutable according to the order of selected nodes\footnote{In this work, we categorize \sortpool~\cite{zhang2018end} as a readout operation (not a hierarchical pooling), because it arranges top-$N'$ nodes in a consistent order by their last hidden dimension. In the PyTorch Geometric library, it is also implemented as one of the \textit{global} methods along with other readout operations.}
or identified node clusters within an input graph, while keeping the structural information in its adjacency matrix $\myov{\amat}$.
For this reason, the node (or node cluster) representations from the graph pooling need to be summarized into a consistent graph-level representation in the end, mainly performed by a \textit{permutation-invariant} readout operation.

In this work, we point out that the global readout operation cannot capture the structural information of each node (or node cluster) representation.
In Figure~\ref{fig:example}, a single graph instance consists of several structurally-meaningful positions, referred to as \textit{structural positions}, whose nodes share distinct local structures.
A global aggregation of all node representations, not explicitly considering such position information, makes the final representation difficult to clearly express the global structure of an input graph.
Particularly, in terms of classification, it allows its classification layer only to adopt global weight parameters that equally consider all the nodes regardless of their structural position.
Since different local structures can differently contribute to target classes, the global weight eventually limits the discrimination power of the classifier.

\section{Structural Semantic Readout}
\label{sec:method}



\subsection{Structural Semantic Readout}
\label{subsec:ssp}
The goal of \proposed is to attain a \textit{structured} graph representation by summarizing the \textit{unordered} node representation.
Formally, it takes the hidden representation $\hmat=[\hvec{1};\ldots;\hvec{N}]\in\mathbb{R}^{N \times d}$ as its input, then outputs $\gmat=[\gvec{1};\ldots;\gvec{K}]\in\mathbb{R}^{K \times d}$ where $K$ is the predefined number of structural positions.
In other words, it maps a node-level (or node cluster-level) representation into a position-level representation that keeps structural consistency for any input graph.
To this end, \proposed (i) aligns each hidden vector with its semantically-closest structural position, then (ii) summarizes the set of hidden vectors aligned with the same position.

\subsubsection{Semantic alignment with structural prototypes}
To obtain the alignment between the nodes and structural positions,
we first define $\mathcal{M}\subset\{0,1\}^{N \times K}$ to be the set of possible binary alignment matrices.
For an alignment matrix $\mmat\in\mathcal{M}$, each entry indicates the alignment between node $n$ and structural position $k$;
i.e., $\mmat_{nk}=1$ if node $n$ is aligned with position $k$, otherwise $\mmat_{nk}=0$.
We additionally enforce the constraint $\sum_k \mmat_{nk}=1$ so that each node should be aligned with only a single position.

Then, we introduce \textit{structural prototypes} $\pmat=[\pvec{1};\ldots,\pvec{K}]\in\mathbb{R}^{K \times d}$, which parameterize the latent semantic (i.e., prototypical hidden features) of the structural positions.
Using a target hidden representation $\hmat$ and the structural prototypes $\pmat$, we can fill the cost matrix $\Delta(\hmat, \pmat)\in\mathbb{R}^{N\times K}$ whose entry becomes the alignment cost (or distance) between each node and structural position.
In this work, we simply use the cosine distance between $\hvec{n}$ and $\pvec{k}$ for the cost function, i.e.,
$[\Delta(\hmat, \pmat)]_{nk}:=\delta(\hvec{n}, \pvec{k}) = 1 - {\hvec{n}\cdot\pvec{k}}/{\lVert\hvec{n}\rVert\lVert\pvec{k}\rVert}$.

Given an arbitrary alignment $\mmat$, the total alignment cost is computed by the inner product between $\mmat$ and $\Delta(\hmat, \pmat)$. 
In this sense, the optimal cost for \textit{structural semantic alignment} (denoted by SSA) and its alignment matrix (denoted by $\mmat^*$) are obtained by
\begin{equation}
\label{eq:align}
\begin{split}
    \text{SSA}(\hmat, \pmat) &= \hspace{6pt} \min_{\mmat\in\mathcal{M}} \langle \mmat, \Delta(\hmat, \pmat) \rangle \\
    \mmat^* &= \argmin_{\mmat\in\mathcal{M}} \langle \mmat, \Delta(\hmat, \pmat) \rangle,
\end{split}
\end{equation}
where $\langle\cdot,\cdot\rangle$ is the Frobenius inner product between two matrices.
Due to the constraint on the alignment matrices, Equation~\eqref{eq:align} can be efficiently solved in a node-wise manner, which are $\text{SSA}(\hmat, \pmat) = \sum_n \min_{k} \delta(\hvec{n}, \pvec{k})$ and $\mmat_{nk}^*=\mathbb{I}[k == \argmin_{k'}\delta(\hvec{n}, \pvec{k'})]$.

\subsubsection{Position-level readout based on node-position alignment}
Using the optimal node-position alignment matrix $\mmat^*$, \proposed generates a summarized vector for structural position $k$:
\begin{equation}
\label{eq:sspool}
\begin{split}
    \gvec{k} &= \phi(\hmat;\mmat^*, k) = \phi(\{\hvec{n}|\mmat^*_{nk}=1, \forall n \in [1, N] \}).  
\end{split}
\end{equation}
$\phi:\mathcal{P}(\mathbb{R}^d)\rightarrow\mathbb{R}^d$ is an aggregation function which maps the set of vectors into a global vector, and $\mathcal{P}$ denotes the powerset of a space.
To be specific, \gmax, \gmean, and \gsum generate a single vector by taking the maximum, mean, and summation value for each latent dimension, respectively. In addition, \gattn~\cite{li2015gated} and \gset~\cite{vinyals2015order} adopt the attention mechanism to take into consideration the importance of each vector, and \gsort~\cite{zhang2018end} performs a 1D spatial convolution along the node vectors sorted by their last hidden dimension.
In the end, $\text{\proposed}(\hmat;\pmat) = [\gvec{1};\ldots;\gvec{K}]$ produces the position-level representation of reduced size $K$.
The overview of \proposed is illustrated in Figure~\ref{fig:arch}.

\begin{theorem}
\label{thm:permuteinv}
\proposed is a permutation-invariant function, whose output is invariant under node permutations.
That is, $\text{\proposed}(\hmat;\pmat)=\text{\proposed}(\pi(\hmat);\pmat)$ for any row-wise matrix permutation $\pi$.
\end{theorem}


\begin{proof}
Let $\mmat^* = \argmin_{\mmat\in\mathcal{M}}\langle \mmat, \Delta(\hmat, \pmat)\rangle$.
Given any row-wise matrix permutation $\pi$ and permutation-invariant aggregation function $\phi$, the $k$-th vector from the \proposed layer is obtained by 
\begin{talign*}
\small
\mmat_\pi^{*} &= \argmin_{\mmat\in\mathcal{M}}\langle \mmat, \Delta(\pi(\hmat), \pmat)\rangle \\
&= \argmin_{\mmat\in\mathcal{M}}\langle \mmat, \pi(\Delta(\hmat, \pmat))\rangle \\
&= \pi(\argmin_{\mmat\in\mathcal{M}}\langle \mmat, \Delta(\hmat, \pmat)\rangle)=\pi(\mmat^*),
\end{talign*}\vspace{-20pt}
\begin{talign*}
\small
[\text{\proposed}(&\pi(\hmat);\pmat)]_k \\
&= \phi(\pi(\hmat);\mmat_\pi^*,k) = \phi(\pi(\hmat);\pi(\mmat^*),k) \\
&= \phi(\hmat; \mmat^*, k) = [\text{\proposed}(\hmat;\pmat)]_k. \qquad \qedhere
\end{talign*}
\end{proof}

\subsubsection{Complexity analyses}
\label{subsubsec:complexity}
The semantic alignment of \proposed takes $O(K)$ for parallel node-wise computation of Equation~\eqref{eq:align}, with additional $KD$ parameters for the structural prototypes $\pmat$;
they do not depend on the number of nodes $N$. 
Since the number of structural positions is set to a smaller value than the number of nodes (i.e., $K < N$), \proposed neither incurs large memory requirements nor computational costs. 

\begin{figure}[t]
	\centering
	\includegraphics[width=\linewidth]{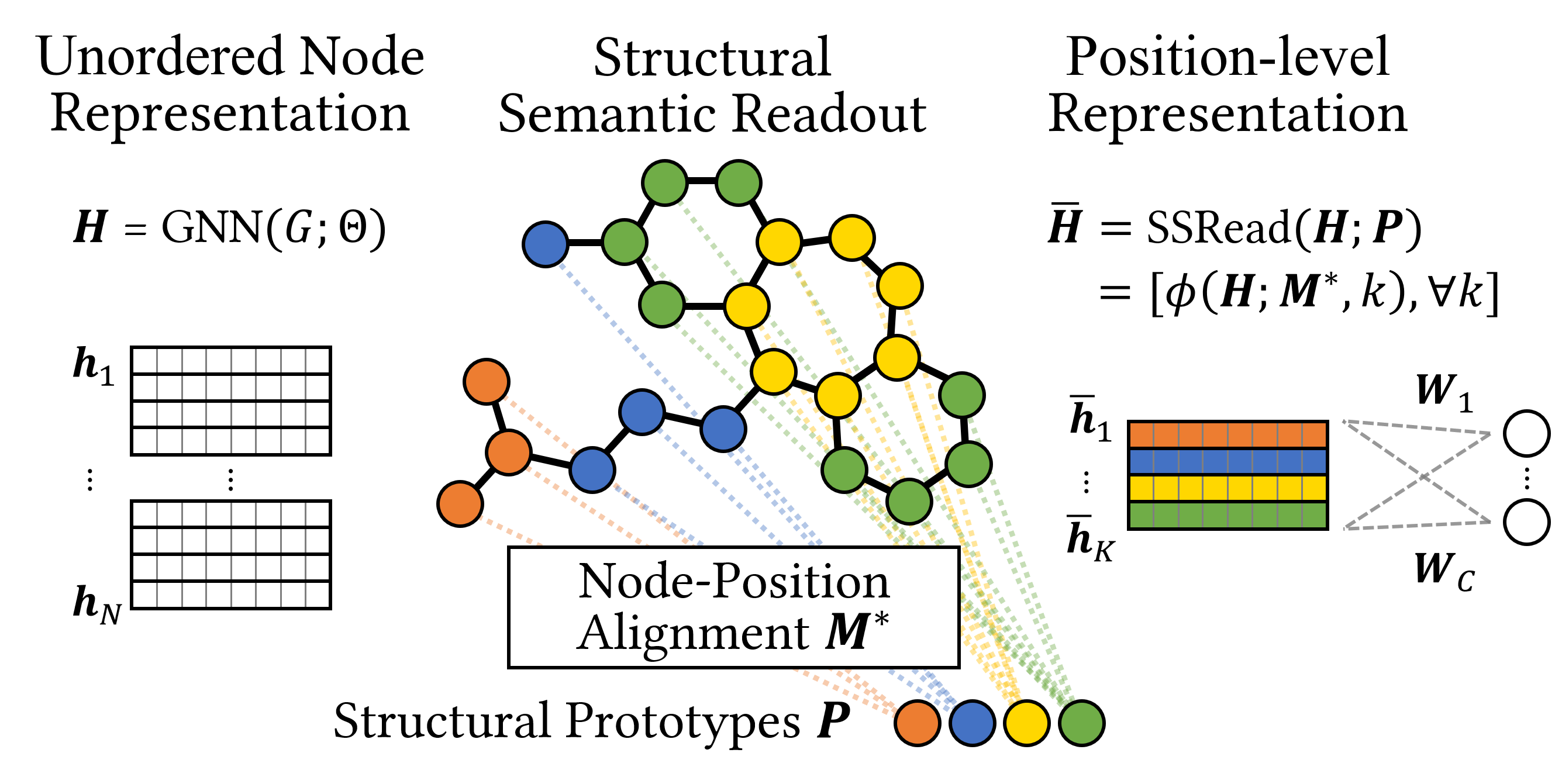}
	\caption{The overview of the structural semantic readout. Based on the optimal alignment between node representations and structural prototypes, \proposed produces the position-level representation. Best viewed in color.}
	\label{fig:arch}
\end{figure}

\subsection{Learning the Structural Prototypes}
\label{subsec:learnprotos}

The remaining challenge is to learn the structural prototypes so that they can capture the latent semantic of $K$ structural positions.
Our key idea is to find the optimal $\pmat$ that minimizes its semantic alignment cost (defined in Equation~\eqref{eq:align}) for the hidden representations of all graph instances in the training set.
However, the naive gradient of Equation~\eqref{eq:align} cannot provide $\pmat$ with useful guidance to learn better alignment, because its hard-min operation takes only the optimal alignment that yields the minimum cost (i.e., not differentiable).
For the effective update of the structural prototypes by using a well-defined gradient, we propose a soft-relaxation of our semantic alignment cost, termed as soft-SSA, by adopting the concept of global alignment kernels~\cite{cuturi2007kernel}.
Instead of the discontinuous hard-min operation, soft-SSA utilizes the soft-min operation with a smoothing parameter $\gamma$ as follows.
\begin{equation}
\label{eq:softssa}
\begin{split}
    \softssa(\hmat, \pmat) 
    &= \text{min}_\gamma\left\{\langle \mmat, \Delta(\hmat, \pmat) \rangle, \forall \mmat\in\mathcal{M}\right\},
\end{split}
\end{equation}
\begin{equation}
\label{eq:softmin}
    \text{min}_\gamma\{a_1, \ldots, a_n\} = 
    \begin{cases}
    \; \min_{i\leq n} a_i, & \gamma = 0\\
    \; -\gamma \log \sum_{i=1}^{n}e^{-a_i/\gamma}, & \gamma > 0.
    \end{cases}
\end{equation}
Note that the original SSA in Equation~\eqref{eq:align} is the special case of soft-SSA with $\gamma=0$.
In contrast to $\text{SSA}(\hmat, \pmat)$, Equation~\eqref{eq:softssa} includes all the alignment costs for possible alignment matrices, with the importance inversely proportional to their alignment cost.
Thus, the gradient of soft-SSA with respect to $\pmat$ (i.e., $\nabla_{\pmat}\softssa(\hmat, \pmat)$) can effectively update $\pmat$ to minimize its distance from $\hmat$
by considering numerous plausible alignments rather than only the optimal alignment.

Using the soft-SSA cost, the alignment loss for learning the representation of the structural prototypes is defined as
\begin{equation}
\label{eq:alignloss}
    \ploss = \frac{1}{|\mathcal{D}|}\sum_{(G,y)\in\mathcal{D}}\softssa(\text{GNN}(G;\Theta),\pmat).
\end{equation}
As the alignment loss decreases, each structural prototype $\pvec{k}$ of the \proposed layer converges to a prototypical hidden vector in the latent space induced by the GNN.
The gradient of Equation~\eqref{eq:alignloss} is efficiently calculated by the following theorem.
\begin{theorem}
\label{thm:softssa}
The soft-relaxation of the semantic alignment cost can be calculated by summing the soft-min values of the alignment cost between each node and structural positions.
That is, $\softssa(\hmat, \pmat) = \sum_{n=1}^{N} \text{min}_\gamma\left\{\delta(\hvec{n}, \pvec{k}), \forall k\in \{1,\ldots, K\}\right\}$.
\end{theorem}
\begin{proof}
For notational simplicity, we define $m_n$ to be the index of the position aligned with node $n$ for an alignment $\mmat\in\mathcal{M}$.
\begin{talign*}
\small
\softssa(\hmat, \pmat) &= -\gamma \log \sum_{\mmat\in\mathcal{M}} \exp(-\frac{1}{\gamma}\langle \mmat, \Delta(\hmat, \pmat) \rangle)\\
&= -\gamma \log \sum_{\mmat\in\mathcal{M}} \exp(-\frac{1}{\gamma}\sum_{n}\delta(\hvec{n},\pvec{m_n}))\\
&= -\gamma \log \prod_{n}\sum_{k}\exp(-\frac{1}{\gamma}\delta(\hvec{n},\pvec{k})) \\
&= \sum_{n} -\gamma \log \sum_{k}\exp(-\frac{1}{\gamma}\delta(\hvec{n},\pvec{k})) \\
&= \sum_{n} \min_{\gamma} \{\delta(\hvec{n},\pvec{k}), \forall k\in\{1,\ldots,K\}\}. \qedhere
\end{talign*}
\end{proof}
It is worth noting that the process of optimizing the structural prototypes does not explicitly require any class labels, because it directly learns from the hidden representations $\hmat=\text{GNN}(G;\Theta)$ of the training graphs.
In other words, $\pmat$ is implicitly guided by a target task used for optimizing the GNN parameters.
For this reason, this objective can be plugged-in a variety of learning frameworks, optimized by supervised or unsupervised (self-supervised) tasks.

\subsection{Learning the GNN Parameters}
\label{subsec:learngnn}
\subsubsection{Supervised learning based on graph class labels}
\label{subsubsec:supervised}
The parameters in a GNN classifier are trained in an end-to-end manner by the supervised loss that guides to learn discriminative graph features among the classes.
For graph classification, the loss is defined based on the cross entropy,
$\closs = -\frac{1}{|\mathcal{D}|}\sum_{(G,y)\in\mathcal{D}}\log P(y|f(G))$
where $P(y|f(G))$ denotes the softmax probability that the final representation of an input graph $f(G)$ belongs to class $y$, usually obtained by a linear layer (or multi-layer perceptrons).

The major advantage of \proposed is the availability of modeling  \textit{position-specific} classification weights,\footnote{In this paper, we use the term ``classification weight'' to represent the weight parameters of the classification layer.} which can be implemented by a fully-connected layer. 
Note that the fully-connected layer combined with \proposed introduces the classification weights independently for each structural position.
As a result, the GNN classifier is able to learn more accurate classification weights for the nodes with different structural semantics.
For a final representation $\gmat=\text{\proposed}(\text{GNN}(G;\Theta);\pmat)$, the softmax probability is defined by
\begin{equation}
\label{eq:softmax}
    P(y=c|\gmat) = \frac{\exp\left(\sum_{k=1}^{K} \gvec{k}\cdot\wvec{k}{c} + b_c \right)}{\sum_{c'=1}^{C}\exp\left(\sum_{k=1}^{K} \gvec{k}\cdot\wvec{k}{c'} + b_{c'}\right)},
\end{equation}
where $\wmat_c=[\wvec{1}{c};\ldots;\wvec{K}{c}]\in\mathbb{R}^{K\times d}$ and $b_c\in\mathbb{R}$ are the classification weight matrix and bias term for class $c$, respectively.
That is, the output score for class $c$ is computed by $s_c = \langle \gmat, \wmat_c \rangle + b_c$. 

\begin{algorithm}[t]
\small
\setstretch{1.2}
\DontPrintSemicolon
    \KwIn{A training set of graph instances $\mathcal{D}$}
	\KwOut{GNN parameters $\Theta$, $\mathcal{W}=\{\wmat_{1},\ldots,\wmat_{C}, \bm{b} \}$, and the structural prototypes $\pmat$}
	\SetKwComment{Comment}{$\triangleright$\ }{} 
	\While{not converged} {
        \For{$(G, y)\in\mathcal{D}$} {
        \nonl \Comment*[r]{Compute the scores for $C$ classes}
        $\hmat \leftarrow \text{GNN}(G;\Theta)$ \;
        $\mmat^* \leftarrow \argmin_{\mmat\in\mathcal{M}}\langle \mmat, \Delta(\pmat, \hmat)\rangle$ \;
        $\gmat \leftarrow [\phi(\hmat;\mmat^*,1);\ldots; \phi(\hmat;\mmat^*,K)]$ \;
        $\svec \leftarrow \text{softmax}([\langle\gmat,\wmat_{1}\rangle+b_1, \ldots, \langle\gmat,\wmat_{C}\rangle+b_C])$ \;
    
        \vspace{3pt}
        \nonl \Comment*[r]{Calculate the two losses}
        $\ploss \leftarrow \softssa(\hmat, \pmat)$ \;
        $\closs \leftarrow \text{CrossEntropy}(y, \svec)$ \;
        
        \vspace{3pt}
        \nonl \Comment*[r]{Update all the model parameters}
        $\pmat \leftarrow \pmat - \eta\cdot{\partial\ploss}/{\partial\pmat}$ \;
        $\Theta \leftarrow \Theta - \eta\cdot {\partial\closs}/{\partial\Theta}$ \;
        $\mathcal{W} \leftarrow \mathcal{W} - \eta\cdot {\partial\closs}/{\partial\mathcal{W}}$ \;
        }
    }
    \caption{Optimization of a GNN classifier}
    \label{alg:pseudocode}
\end{algorithm}

\subsubsection{Self-supervised learning based on augmented graphs}
\label{subsubsec:selfsupervised}
In addition to supervised learning, we present how to learn the GNN graph encoder that adopts the \proposed layer in an unsupervised manner. 
Inspired by the recent success of self-supervised learning in computer vision~\cite{chen2020simple} and natural language processing~\cite{devlin2019bert}, several self-supervised methods for graph representation learning were developed~\cite{rong2020self,you2020graph} and they showed remarkable results for many downstream tasks.
Among them, the graph-level contrastive loss~\cite{you2020graph} can be utilized instead of the classification loss for training our graph encoder, because it provides supervision on graph-level representations based on the relationship among graphs.

Precisely, graph contrastive learning learns the graph representations of positively-related (i.e., similar) instances to be close, while those of negatively-related (i.e., dissimilar) ones far from each other.
To identify similar and dissimilar graph pairs without any class labels, it utilizes the graph augmentation function $T$ which stochastically adds small noises to an input graph (e.g., random node drop, edge perturbation).
For a target graph $G$, any augmented graphs from $T(G)$ are regarded as \textit{positive}, and all the other graphs in the training set are used as \textit{negative}.
That is, the contrastive loss for graph representation learning is described by
\begin{equation}
\begin{split}
\label{eq:contrloss}
    &\sloss = \\ &-\frac{1}{|\mathcal{D}|}\sum_{G\in\mathcal{D}} \expect_{G_+\sim\mathcal{T}(G)}\log\frac{\exp\left(\text{sim}(G, G_+)\right)}{\sum_{G_-\in\mathcal{D}\backslash\{G\}}\exp\left(\text{sim}(G, G_-)\right)}.
\end{split}
\raisetag{43pt}
\end{equation}
The sim function is defined by the cosine similarity between the projected representations of two graphs;
given a GNN graph encoder $f$ and a projection function $g$,
$\text{sim}(G_1, G_2)=\zvec{1}\cdot\zvec{2}/\lVert\zvec{1}\rVert\lVert\zvec{1}\rVert$ where $\zvec{1}=g(f(G_1))$ and $\zvec{2}= g(f(G_2))$.
To exploit the position-level representation obtained by the \proposed layer, we flatten the representation by $f(G)=(\text{Flatten}\circ\text{\proposed}\circ\text{GNN})(G)$.
For efficiency and scalability, only the graphs in each minibatch are considered as negative for the computation of Equation~\eqref{eq:contrloss}, as done in~\cite{you2020graph, chen2020simple}.

\subsubsection{Optimization}
\label{subsubsec:optimization}
Algorithm~\ref{alg:pseudocode} describes the overall process of training the GNN classifier that adopts the \proposed layer.
As training progresses, the GNN parameters $\Theta$ and the structural prototypes $\pmat$ collaboratively improve with each other.
To be precise, updating the GNN parameters results in better node representations $\hmat=\text{GNN}(G;\Theta)$ that encodes the high-level features of local graph structures, thereby $\pmat$ can learn the prototypical features from the enhanced representations.
In addition, using more accurate structural prototypes for the \proposed layer encourages the GNN to learn further discriminative features because the classification loss is computed from the position-level representation $\gmat=\text{\proposed}(\hmat;\pmat)$. 

In case of unsupervised learning, the GNN parameters except for the classification weights can be updated by the contrastive loss $\sloss$ (defined in Equation~\eqref{eq:contrloss}), instead of the classification loss $\closs$.
Without any graph labels, the GNN graph encoder is effectively trained to output useful graph representations, while its \proposed layer is optimized to identify $K$ structural positions based on the semantic alignment with the structural prototypes.

\section{Experiments}
\label{sec:exp}


\subsection{Experimental Settings}
\label{subsec:expset}
\subsubsection{Datasets}
In our experiments, we use 6 graph classification benchmarks from the biochemical domain (i.e., \dnd, \mutag, \mutagen, \nci, and \proteins) and social domain (i.e., \imdbb and \imdbm), collected in TU datasets~\cite{morris2020tudataset}.
For in-depth analyses on molecular graphs, we utilize the additional data sources (e.g., SMILES) of \mutag downloaded from the external ChemDB.\footnote{http://cdb.ics.uci.edu/cgibin/LearningDatasetsWeb.py}
Figure~\ref{fig:example} shows several examples of graph instances from \mutag and \imdbb, respectively.

\subsubsection{Baselines}
\label{subsubsec:baseline}
We validate the effectiveness of our \proposed layer against the global readout layer (denoted by \textbf{\gread}) by using various GNN architectures.
As the main architecture, we employ the GNN classifier that consists of three \gcn layers~\cite{kipf2017semi} without any hierarchical pooling layers (denoted by \textbf{\gcn}).
We also consider the state-of-the-art GNN classifiers that use advanced message passing or hierarchical pooling.
\begin{itemize}
    \item \textbf{\gin}~\cite{xu2018powerful}: The graph isomorphism network that tailors the GNN to be injective for maximizing its representational capacity as powerful as the Wisfeiler-Lehman (WL) test.
    \item \textbf{\dgcnn}~\cite{zhang2018end}: The GNN that adopts 1D conv layers after sorting top-$N'$ nodes by their last hidden dimension.
    \item \textbf{\diffpool}~\cite{ying2018hierarchical}: The GNN that adopts differentiable graph pooling layers based on the soft cluster assignment.
    \item \textbf{\sagpool}~\cite{lee2019self}: The GNN with pooling layers that keep only a portion of nodes based on their attention scores.
    \item \textbf{\gunet}~\cite{gao2019graph}: The U-Net architecture that downscales a graph by selecting nodes based on their projected scores.
\end{itemize}
Note that the goal of our work is to enhance the discrimination power of existing GNN classifiers by replacing their \gread layer with the \proposed layer.
For this reason, we follow the details of each architecture (i.e., the composition of GNN layers and its classification module) provided by the original papers.\footnote{Since the property of various message passing and hierarchical pooling is all different, their optimal GNN architecture also cannot be the same. For this reason, we employ the entire architecture suggested by the original papers, rather than simply replacing each GNN layer in a single fixed architecture.}
The detailed architectures are described in Table~\ref{tbl:gnnarchs}.
We remark that any other hierarchical graph pooling techniques~\cite{noutahi2019towards, yuan2020structpool}, which are not included in our experiments, are also compatible with our readout layer;
this means that their final classification performance can be further improved with the help of \proposed.

In addition to the supervised task that learns from labeled graphs, we also evaluate our readout layer optimized in an unsupervised setting.
As several unsupervised methods for node-level representation learning~\cite{sun2019infograph, velivckovic2018deep} can be applied to graph classification, we consider one of them as well as the self-supervised method for graph-level representation learning~\cite{you2020graph}.
Both of them necessarily use any readout, which can be either \gread or \proposed.
\begin{itemize}
    \item \textbf{\graphim}~\cite{sun2019infograph}: The graph learning framework that maximizes the mutual information between graph-level and subgraph-level (of various scales) representations.
    \item \textbf{\graphcl}~\cite{you2020graph}: The graph contrastive learning framework based on graph augmentation (Section~\ref{subsubsec:selfsupervised}).
\end{itemize}

\begin{table}[t]
\caption{Statistics of the datasets.}
\label{tbl:datastats}
\centering
\resizebox{0.99\linewidth}{!}{%
\begin{tabular}{rcccc}
    \hline
    \textbf{Datasets} & \textbf{\#Graphs} & \textbf{Avg.\#Nodes} & \textbf{Avg.\#Edges} & \textbf{\#Classes} \\\hline
    \dnd & 1,178 & 284.32 & 715.66 & 2 \\
    \mutag & \ \ 188 & \ \ 17.93 & \ \ 19.79 & 2 \\
    \mutagen & 4,337 & \ \ 30.32 & \ \ 30.77 & 2 \\
    \nci & 4,110 & \ \ 29.87 & \ \ 32.30 & 2 \\    
    \proteins & 1,113 & \ \ 39.06 & \ \ 72.82 & 2 \\\hline
    \imdbb & 1,000 & \ \ 19.77 & \ \ 96.53 & 2 \\    
    \imdbm & 1,500 & \ \ 13.00 & \ \ 65.94 & 3 \\    
    \hline
\end{tabular}
}
\end{table}

\subsubsection{Implementation Details}
\label{subsubsec:implementation}
We implement all the GNN classifiers, including the \gread and \proposed layers, by using PyTorch\footnote{https://pytorch.org/} and PyTorch Geometric.\footnote{https://pytorch-geometric.readthedocs.io/en/latest/}
For global aggregation functions, we employ the official implementation of \gsum, \gmax, \gmean, \gattn, \gset, and \gsort, provided by Pytorch Geometric.
In case of \gsort, we tailor its aggregation module from $(\text{\gsort}:\mathcal{P}(\mathbb{R}^d)\rightarrow\mathbb{R}^{N'\times d})$ that outputs sorted node representations to $(\text{Flatten}\circ\text{Conv}\circ\text{\gsort}:\mathcal{P}(\mathbb{R}^d)\rightarrow\mathbb{R}^d)$ that outputs a summarized representation, in order to facilitate to combine it with our \proposed layer.
We only consider the global aggregation functions officially implemented in PyTorch Geometric, but any other global aggregations, such as the second-order pooling~\cite{wang2020second}, also can be combined with our \proposed layer to output the position-level representation.

\subsubsection{Experimental settings}
\label{subsubsec:expset}
For quantitative evaluation on graph classification, we follow the fair evaluation setup suggested by~\cite{errica2019fair}.
In detail, we train each GNN classifier with 10-fold cross validation, while we further partition the training data into training/validation sets with the ratio of 9:1.
The maximum number of epochs is set to 500, and we stop training if the performance on the validation set does not improve for consecutive 50 epochs.
We repeat each experiment five times with different random seeds and report their mean performance with the standard deviation.
All the experiments are conducted on NVIDIA Titan Xp for GPU parallel computation.
The elapsed time for training the GNN classifiers (Figure~\ref{fig:trainingtime}) is also measured in this environment.


For each dataset, we choose the optimal number of structural positions $K\in\{2, 4, 8, 16\}$ that achieves the best validation accuracy for the base classifier, i.e., \gcn + \proposed (\gsum).
In case of the smoothing parameter $\gamma$, we fix its value to 0.01 without the search to eliminate the benefit from hyperparameter tuning.
The sensitivity analysis on these hyperparameters is provided in Section~\ref{subsec:sensitivity}.

\begin{table*}[t]
\caption{The detailed architecture of GNN classifiers. ``MP'': message passing layer, ``Pool'': hierarchical pooling layer.
All the layers (i.e., MP, Pool, and Readout) are optimized by target supervised/self-supervised tasks in an end-to-end way.}
\label{tbl:gnnarchs}
\centering
\resizebox{0.99\linewidth}{!}{%
\begin{tabular}{ccccccc}
    \hline
    \textbf{GNN} & \textbf{Message} & \textbf{Hierarchical} & \textbf{Pooling} & \textbf{Architecture} & \multirow{2}{*}{\textbf{Readout}} & \multirow{2}{*}{\textbf{Classification}} \\
    \textbf{Classifier} & \textbf{Passing} & \textbf{Pooling} & \textbf{Ratio} & \textbf{(Before Readout)} & & \\\hline
    \gin~\cite{xu2018powerful} & \gin~\cite{xu2018powerful} & - & -  & MP+MP+MP+MP+MP & \gsum & Linear \\
    \dgcnn~\cite{zhang2018end} & \gcn~\cite{kipf2017semi} & - & - & MP+MP+MP+MP & \gsort(+Conv) & 2L MLP \\
    \diffpool~\cite{ying2018hierarchical} & SAGE~\cite{hamilton2017inductive} & \diffpool~\cite{ying2018hierarchical} & 0.1 & (MP+MP+MP)+Pool+MP & \gmax & 2L MLP \\\
    \sagpool~\cite{lee2019self} & \gcn~\cite{kipf2017semi} & \sagpool~\cite{lee2019self} & 0.5,0.5,0.5 & MP+Pool+MP+Pool+MP+Pool & \gmax,\gmean & 3L MLP \\
    \gunet~\cite{gao2019graph} & \gcn~\cite{kipf2017semi} & \topkpool~\cite{gao2019graph} & 0.9,0.8,0.7 & {\small (MP+Pool+MP+Pool+MP+Pool)+(MP+MP+MP)} & \gsum,\gmax,\gmean & 2L MLP \\    
    \hline
\end{tabular}
}
\end{table*}

\begin{table*}[t]
\caption{Classification accuracy of the base GNN classifier with \gread and \proposed, using five different aggregation functions. The results of statistically-significant improvement (i.e., $p \leq 0.05$ from the paired $t$-test) are marked in bold face.}
\label{tbl:gpoolresults}
\centering
\resizebox{0.99\linewidth}{!}{%
\begin{tabular}{Lccccccc}
    \hline
    \textbf{GNN Architecture} & \textbf{\dnd} & \textbf{\mutag} & \textbf{\mutagen} & \textbf{\nci} & \textbf{\proteins} & \textbf{\imdbb} & \textbf{\imdbm} \\\hline
    GCN + \gread (\gsum) & 67.25 (0.92) & 84.84 (1.24) & 80.11 (0.25) & 79.79 (0.30) & 71.80 (0.80) & 70.62 (0.49) & 47.20 (0.66) \\
    GCN + \proposed (\gsum) & \textbf{70.23} (1.03) & \textbf{87.16} (0.84) & \textbf{81.37} (0.58) & \textbf{81.60} (0.42) & \textbf{73.55} (0.57) & \textbf{71.18} (0.45) & \textbf{47.81} (0.60) \\\hline
    GCN + \gread (\gmax) & 73.87 (0.99) & 85.54 (1.61) & 78.91 (0.41) & 71.08 (0.88) & 67.37 (0.67) & 70.48 (0.64) & 46.63 (0.65) \\
    GCN + \proposed (\gmax) & \textbf{75.86} (0.61) & \textbf{87.87} (1.19) & \textbf{80.32} (0.50) & \textbf{73.19} (0.83) & \textbf{69.02} (1.08) & \textbf{71.26} (0.46) & \textbf{47.73} (0.30) \\\hline
    GCN + \gread (\gmean) & 65.77 (0.28) & 84.63 (1.20) & 79.64 (0.53) & 78.47 (0.78) & 68.86 (0.67) & 70.72 (0.47) & 46.89 (0.28) \\
    GCN + \proposed (\gmean) & \textbf{66.71} (0.88) & \textbf{86.30} (1.25) & \textbf{81.86} (0.59) & \textbf{80.88} (0.31) & \textbf{71.35} (0.77) & 70.94 (0.74) & \textbf{48.24} (0.32) \\\hline
    GCN + \gread (\gattn) & 70.68 (0.79) & 84.37 (0.82) & 80.85 (0.25) & 78.26 (0.28) & 70.62 (0.42) & 70.98 (0.56) & 47.05 (0.33) \\
    GCN + \proposed (\gattn) & 71.03 (0.54) & \textbf{85.35} (0.56) & \textbf{81.27} (0.33) & \textbf{80.83} (0.41) & \textbf{71.41} (0.85) & \textbf{71.66} (0.54) & \textbf{47.99} (0.66) \\\hline
    GCN + \gread (\gset) & 72.26 (0.63) & 84.61 (0.90) & 80.37 (0.58) & 78.29 (0.50) & 70.13 (0.81) & 70.94 (0.56) & 47.31 (0.40) \\
    GCN + \proposed (\gset) & \textbf{73.14} (0.82) & \textbf{87.37} (1.56) & \textbf{81.14} (0.68) & \textbf{80.56} (0.28) & 71.23 (1.12) & 71.06 (0.26) & \textbf{48.13} (0.13) \\\hline
    GCN + \gread (\gsort) & 71.03 (0.75) & 83.85 (1.07) & 79.04 (0.35) & 78.09 (0.44) & 71.56 (0.55) & 70.56 (0.27) & 47.76 (0.66) \\
    GCN + \proposed (\gsort) & \textbf{72.24} (0.76) & \textbf{86.48} (1.05) & \textbf{81.46} (0.22) & \textbf{80.44} (0.35) & \textbf{73.78} (0.19) & \textbf{71.26} (0.51) & \textbf{48.53} (0.34) \\\hline
\end{tabular}
}
\end{table*}

\subsection{Effectiveness of \proposed}
\subsubsection{The effectiveness of \proposed with various aggregation functions}
\label{subsubsec:comppool}
To validate the effectiveness of our \proposed, we first compare the performance of the base GNN classifier that adopts either \gread or \proposed.
We consider six different aggregation functions $\phi$:
\gsum, \gmax, \gmean, \gattn, \gset, and \gsort.
The aggregation functions that employ parametric modules (i.e., \gattn, \gset, and \gsort) are also optimized for the target task.
In Table~\ref{tbl:gpoolresults}, \proposed shows significantly higher classification accuracy than \gread for all the datasets.
In case of \gread, the best performing aggregation function varies depending on the target dataset, while \proposed further improves its performance independently with the type of aggregation. 
Particularly, for all the datasets except for \dnd, \proposed (with \gsum) outperforms \gread (with its best performing $\phi$).
The results demonstrate good compatibility of \proposed with any aggregation functions used for its position-level readout operation.

\subsubsection{The effectiveness of \proposed with the state-of-the-art GNN classifiers}
\label{subsubsec:compgnn}
\begin{table*}[t]
\caption{Classification accuracy of the state-of-the-art GNN classifiers with \gread and \proposed. The results of statistically-significant improvement (i.e., $p \leq 0.05$ from the paired $t$-test) are marked in bold face.}
\label{tbl:gnnresults}
\centering
\resizebox{0.99\linewidth}{!}{%
\begin{tabular}{Lccccccc}
    \hline
    \textbf{GNN Architecture} & \textbf{\dnd} & \textbf{\mutag} & \textbf{\mutagen} & \textbf{\nci} & \textbf{\proteins} & \textbf{\imdbb} & \textbf{\imdbm} \\\hline
    \gin + \gread(\gsum) & 70.81 (0.94) & 86.09 (0.23) & 82.30 (0.34) & 80.71 (0.34) & 71.28 (0.48) & 70.44 (0.59) & 47.23 (0.55)\\
    \gin + \proposed(\gsum) & \textbf{72.24} (0.54) & \textbf{87.22} (0.65) & \textbf{82.73} (0.31) & \textbf{82.90} (0.33) & \textbf{73.14} (0.44) & \textbf{71.64} (0.52) & \textbf{47.93} (0.37) \\
    \hline
    \dgcnn + \gread(\gsort) & 73.57 (0.66) & 82.43 (0.67) & 74.94 (0.54) & 72.86 (0.46) & 71.66 (0.69) & 70.48 (0.70) & 47.06 (0.59) \\
    \dgcnn + \proposed(\gsort) & \textbf{74.72} (0.84) & \textbf{84.77} (0.75) & \textbf{77.15} (0.48) & \textbf{75.66} (0.65) & \textbf{72.79} (0.78) & 70.78 (0.94) & \textbf{47.83} (0.42)\\
    \hline
    \diffpool + \gread(\gmax) & 76.27 (1.04) & 80.22 (0.77) & 81.62 (0.42) & 80.02 (0.14) & 71.79 (0.50) & 63.62 (1.01) & 45.60 (0.45) \\
    \diffpool + \proposed(\gmax) & 77.23 (0.97) & \textbf{82.20} (0.66) & 81.52 (0.36) & \textbf{80.54} (0.19) & 71.21 (0.68) & \textbf{64.66} (0.66) & \textbf{46.72} (0.39) \\
    \hline
    \sagpool + \gread(\gmax, \gmean) & 73.65 (0.44) & 74.89 (0.78) & 76.04 (0.53) & 71.50 (0.56) & 71.50 (0.71) & 55.56 (0.99) & 39.10 (0.23) \\
    \sagpool + \proposed(\gmax, \gmean) & \textbf{75.74} (0.80) & \textbf{76.06} (0.79) & 75.72 (0.73) & \textbf{72.61} (0.61) & \textbf{73.05} (0.81) & \textbf{58.30} (0.86) & \textbf{39.93} (0.42) \\
    \hline
    \gunet + \gread(\gsum, \gmax, \gmean) & 74.43 (0.35) & 83.29 (0.75) & 81.84 (0.55) & 79.27 (0.26) & 75.42 (0.27) & 57.94 (1.03) & 43.33 (0.70) \\
    \gunet + \proposed(\gsum, \gmax, \gmean) & \textbf{76.96} (0.57) & \textbf{84.34} (0.56) & 81.99 (0.54) & \textbf{79.71} (0.32) & 75.20 (0.36) & \textbf{59.16} (0.79) & \textbf{44.47} (0.99) \\
    \hline
\end{tabular}
}
\end{table*}
\begin{table*}[t]
\caption{Classification accuracy on the graph-level representations from \gread and \proposed, optimized in an unsupervised setting. The results of statistically-significant improvement (i.e., $p \leq 0.05$ from the paired $t$-test) are marked in bold face.}
\label{tbl:selfsupresults}
\centering
\resizebox{0.99\linewidth}{!}{%
\begin{tabular}{lZcccccccc}
    \hline
    \textbf{Learning} & \textbf{GNN Architecture} & \textbf{\dnd} & \textbf{\mutag} & \textbf{\mutagen} & \textbf{\nci} & \textbf{\proteins} & \textbf{\imdbb} & \textbf{\imdbm} \\\hline
    \multirow{2}{*}{\graphim} 
    & \gin + \gread (\gsum) & 70.46 (0.84) & 84.91 (0.61) & 73.93 (0.55) & 70.54 (0.46) & 72.94 (0.78) & 68.66 (0.49) & 47.52 (0.34) \\
    & \gin + \proposed (\gsum) & \textbf{72.51} (0.81) & \textbf{86.73} (0.68) & \textbf{77.20} (0.51) & \textbf{73.73} (0.43) & \textbf{74.00} (0.30) & \textbf{69.52} (0.28) & \textbf{49.31} (0.28) \\\hline
    
    \multirow{2}{*}{\graphcl} 
    & \gin + \gread (\gsum) & 71.47 (0.28) & 85.61 (0.65) & 74.67 (0.08) & 70.65 (0.96) & 73.07 (0.52) & 70.58 (0.27) & 47.76 (0.42) \\
    & \gin + \proposed (\gsum) & \textbf{72.94} (0.45) & \textbf{88.21} (0.73) & \textbf{78.45} (0.15) & \textbf{74.77} (0.68) & \textbf{74.19} (0.27) & \textbf{72.12} (0.34) & \textbf{49.52} (0.30) \\
    \hline
\end{tabular}
}
\end{table*}
Next, we compare the performance of the state-of-the-art GNN classifiers that adopt either \gread or \proposed. 
For the readout of each classifier, we adopt the aggregation function $\phi$ suggested in the original papers:
they simply utilize one of the aggregators (i.e., \gsum, \gmax, \gmean, \gsort)~\cite{xu2018powerful, lee2019self, zhang2018end}, or use multiple types of aggregation functions together to concatenate the obtained representations~\cite{ying2018hierarchical, gao2019graph}. 
In Table~\ref{tbl:gnnresults}, \proposed achieves the consistent improvement over \gread for most cases.
Even though the existing GNN classifiers adopt several pooling layers for hierarchical graph summarization, their last \gread layer incurs information loss on the global structure of an input graph.
Thus, our \proposed layer successfully boosts the performance by additionally leveraging the global structural information, as a completely independent module with other hierarchical graph pooling layers.

\subsubsection{The effectiveness of \proposed in an unsupervised setting}
\label{subsubsec:compselfsup}
We also investigate the effectiveness of \proposed in the unsupervised setting, where only unlabeled graphs are available for training the GNN graph encoder.
To this end, we first optimize the GIN, which adopts either \gread or \proposed as its readout layer, based on the unsupervised (or self-supervised) learning frameworks.
Then, we measure the performance of a linear SVM classifier trained on their graph-level representations with the class labels, based on 10-fold cross validation; 
it is also known as \textit{linear evaluation} protocol conventionally used to assess the quality of representations in the unsupervised setting~\cite{velivckovic2018deep, sun2019infograph, you2020graph, rong2020self}.
In Table~\ref{tbl:selfsupresults}, \proposed brings a significant improvement over \gread for both the cases of \graphim and \graphcl.
Interestingly, their performances are comparable to that of the supervised classifiers (Table~\ref{tbl:gnnresults}, \gin) especially for the datasets containing only a small number of graphs (e.g., \dnd, \mutag, and \proteins).
This is because the supervised GNN classifiers are easily overfitted to the limited number of training graphs and their class labels, whereas the unsupervised GNN encoders are less affected by the class labels and encode richer semantic related to graph structures.
In conclusion, \proposed successfully learns the structural prototypes from the unsupervised tasks as well, 
which are as effective as the supervised tasks for several datasets, thereby it enhances the quality of graph-level representation.

\begin{figure}[t]
	\centering
	\includegraphics[width=\linewidth]{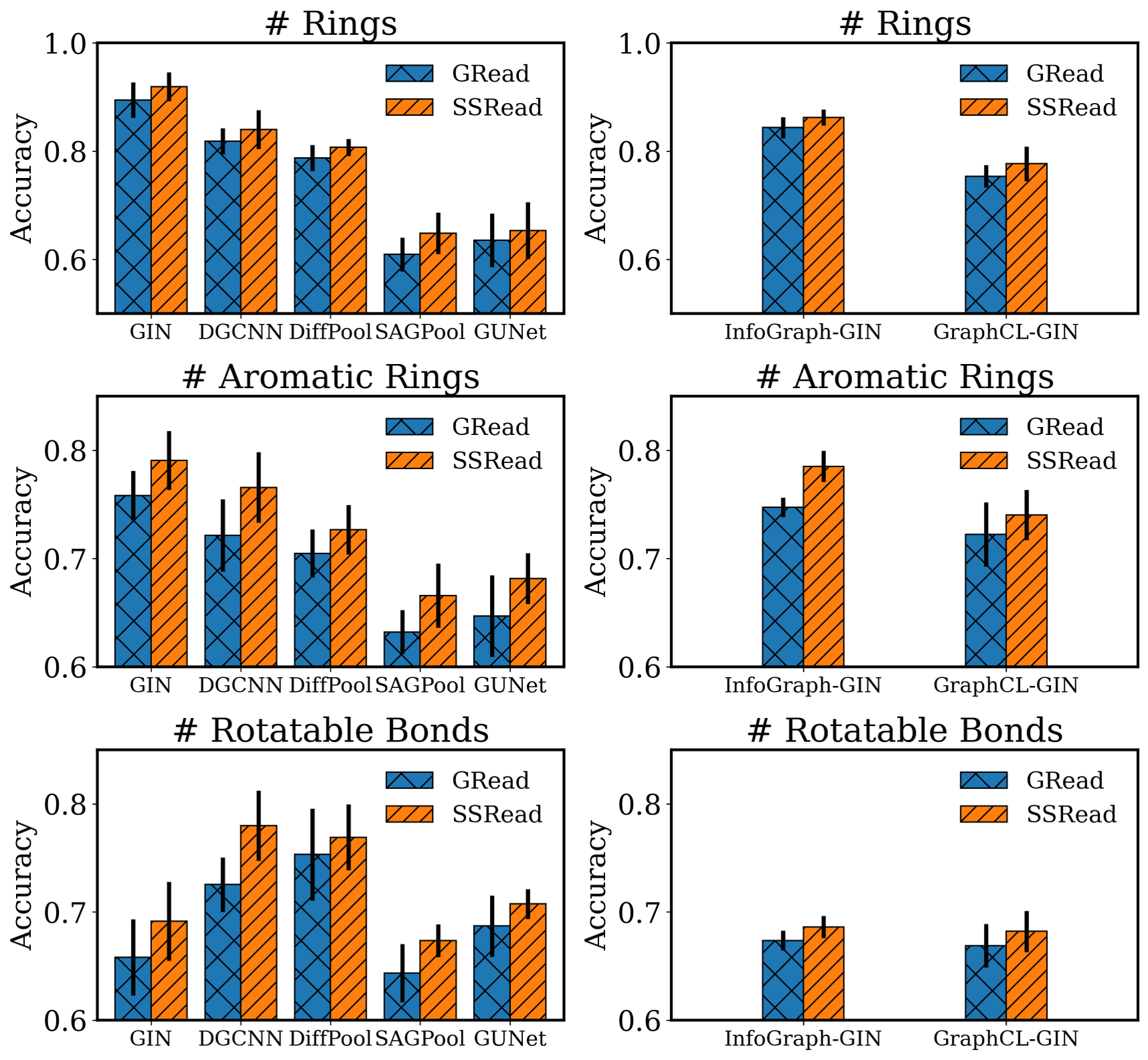}
	\caption{Performance of \gread and \proposed in predicting global structural properties, optimized in supervised (Left) and unsupervised (Right) settings. Dataset: \mutag.}
	\label{fig:prediction}
\end{figure}

\subsubsection{The effectiveness of \proposed in predicting global structural properties}
\label{subsubsec:compmp}
To evaluate how effectively the structural information is encoded into the graph-level representation, we introduce new tasks of predicting several structural properties\footnote{They are extracted from SMILES of each molecule by using the rdkit.Chem package.} (i.e., the number of rings, aromatic rings, and rotatable bonds) on the \mutag dataset.
Similar to Section~\ref{subsubsec:compselfsup}, we measure the accuracy of a linear SVM classifier based on 10-fold cross validation, after optimizing the GNN encoder in the supervised or unsupervised manner.
In Figure~\ref{fig:prediction}, \proposed significantly beats \gread for all the cases, implying that it can effectively capture the structural properties into the final representations.
There are two minor observations on the comparison results.
First, the supervised \gin performs better than the unsupervised \gin in case of predicting the number of rings.
This is because there is a strong correlation between the number of rings in each molecule (the target structural property) and its mutagenic activity (the class label)~\cite{debnath1991structure}.\footnote{In \mutag, the molecules possessing one or two fused rings have much less mutagenic potency than compounds with three or more fused rings~\cite{debnath1991structure}.}
%
Second, \graphim performs better than \graphcl for these tasks (Figure~\ref{fig:prediction}, Right), which indicates that supervision from augmented graphs in \graphcl is not enough to learn the global structural properties.
Despite these dynamics of performances, our \proposed steadily strengthens their capability to capture the global structures compared to the case of using the \gread, since every GNN architecture requires a readout operation to obtain graph-level representations.

\begin{figure}[t]
	\centering
	\includegraphics[width=\linewidth]{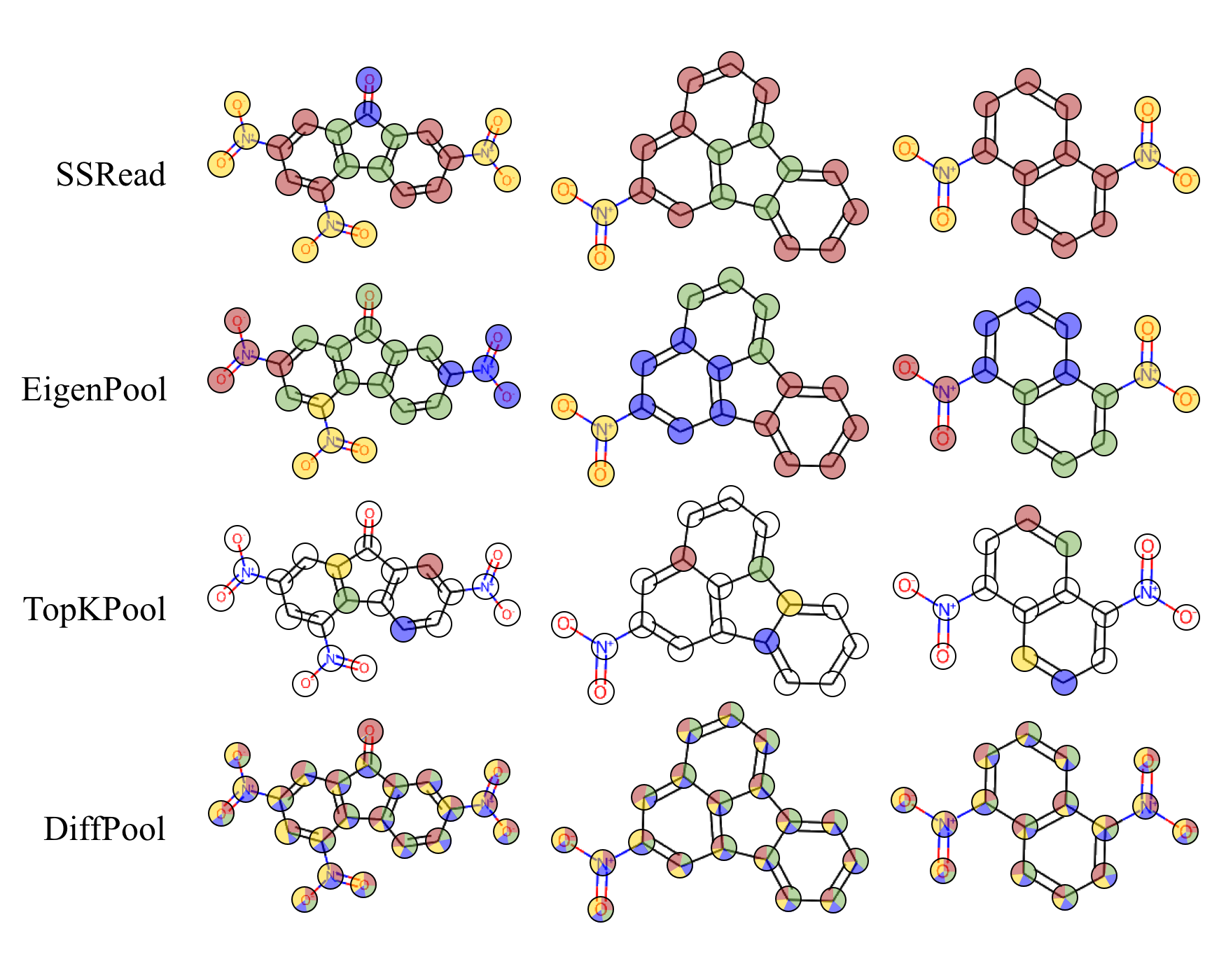}
	\caption{Visualization of our structural semantic alignment ($K=4$) and other hierarchical graph poolings ($N'=4$). The structural positions and node-clusters (or selected nodes) are represented in different colors. Dataset: \mutag.}
	\label{fig:analysis}
\end{figure}

\subsection{Qualitative Analysis}
\label{subsec:qualitative}
We visualize the node-position alignment obtained by \proposed ($K=4$), so as to compare it with the result of node clustering or node selection obtained by existing graph poolings ($N'=4$).
In Figure~\ref{fig:analysis}, three molecular graphs from \mutag are drawn with their nodes marked in different colors.
\proposed clearly identifies consistent structural positions, each of which actually corresponds to a different functional group derived by chemical knowledge:
\textit{benzene rings} (red), \textit{carbonyl group} (blue), \textit{nitro group} (yellow), and \textit{cyclo-pentadiene} (green).
On the other hand, \eigenpool~\cite{ma2019graph} that performs spectral clustering partitions each graph into $N'$ connected subgraphs, but each subgraph (or node cluster) does not imply any consistent chemical semantic;
the same functional groups are colored differently across different molecules as well as within a single molecule.
\topkpool~\cite{gao2019graph} selects only $N'$ nodes within benzene rings while dropping all the other nodes, and the soft cluster assignment of \diffpool~\cite{ying2018hierarchical} is hard to be interpreted or matched with any functional groups.
In conclusion, our \proposed is capable of performing semantic segmentation on input graphs without using explicit labels of structural semantic, by learning the structural prototypes from the training dataset.
This can be practically used for discovering meaningful substructures from graph datasets.

We also qualitatively analyze the class activation map (CAM)~\cite{zhou2016learning, pope2019explainability} on test graphs, which provides an explanation on how much each structural region (i.e., node) contributes to the classification of a target graph.
Similar to grad-CAM~\cite{selvaraju2017grad},
the activation of node $n$ for its target class $c$ is defined by $\text{CAM}_c(n) = ({\partial s_c}/{\partial \hvec{n}})\cdot \hvec{n}$, where the node-wise gradient $(\partial s_c / \partial \hvec{n}) \in \mathbb{R}^{d}$ implies the position-specific classification weight modeled by Equation~\eqref{eq:softmax}.
Figure~\ref{fig:camscore} presents correctly-classified test graphs (whose computed softmax probabilities are larger than 0.9), with their nodes highlighted proportionally to their CAM scores.
Note that the class label of \mutag is highly correlated with the number of rings in a molecule, as discussed in Section~\ref{subsubsec:compmp}.
From the perspective of this chemical knowledge, \proposed accurately localizes the ring structure as the discriminative region of class 1, whereas \gread more focuses on the rest of the ring when predicting class 1.
Since \gread aggregates all node representations regardless of their structural position, it identifies the nitro and carbonyl groups as the representative structure for class 1, which are clearly distinguished from the ring structure of class 2.
In contrast, our position-level readout allows to extract more discriminative features within the ring structure (precisely, its corresponding structural position), while keeping the importance of other structures lower.
This analysis supports that the position-specific classification weights help to improve the localization performance of the GNN classifier.

\begin{figure}[t]
	\centering
	\includegraphics[width=\linewidth]{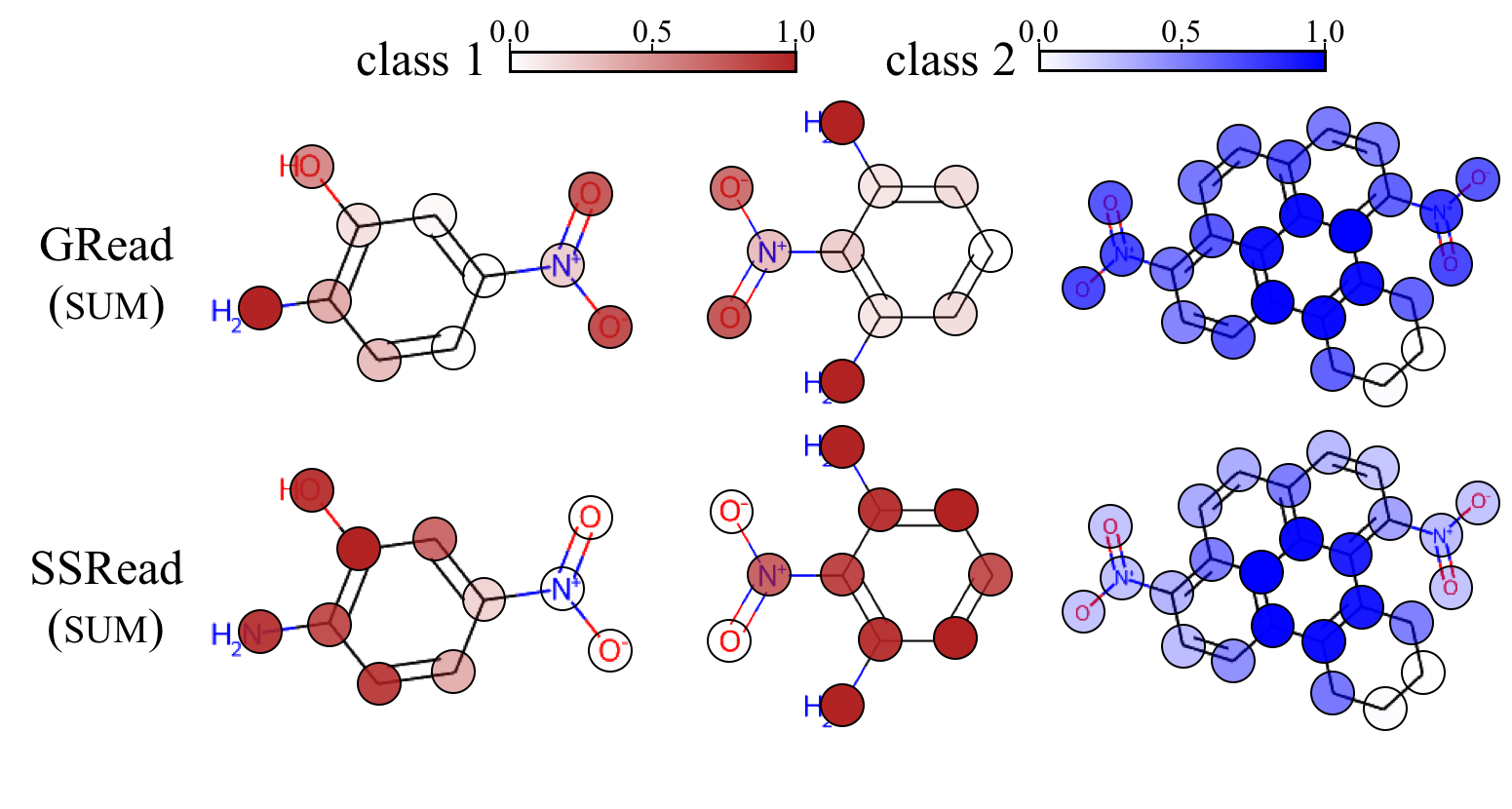}
	\caption{Test graphs from class 1 (Red) and class 2 (Blue), highlighted with their CAM scores. Dataset: \mutag.}
	\label{fig:camscore}
\end{figure}
\begin{figure}[t]
	\centering
	\includegraphics[width=\linewidth]{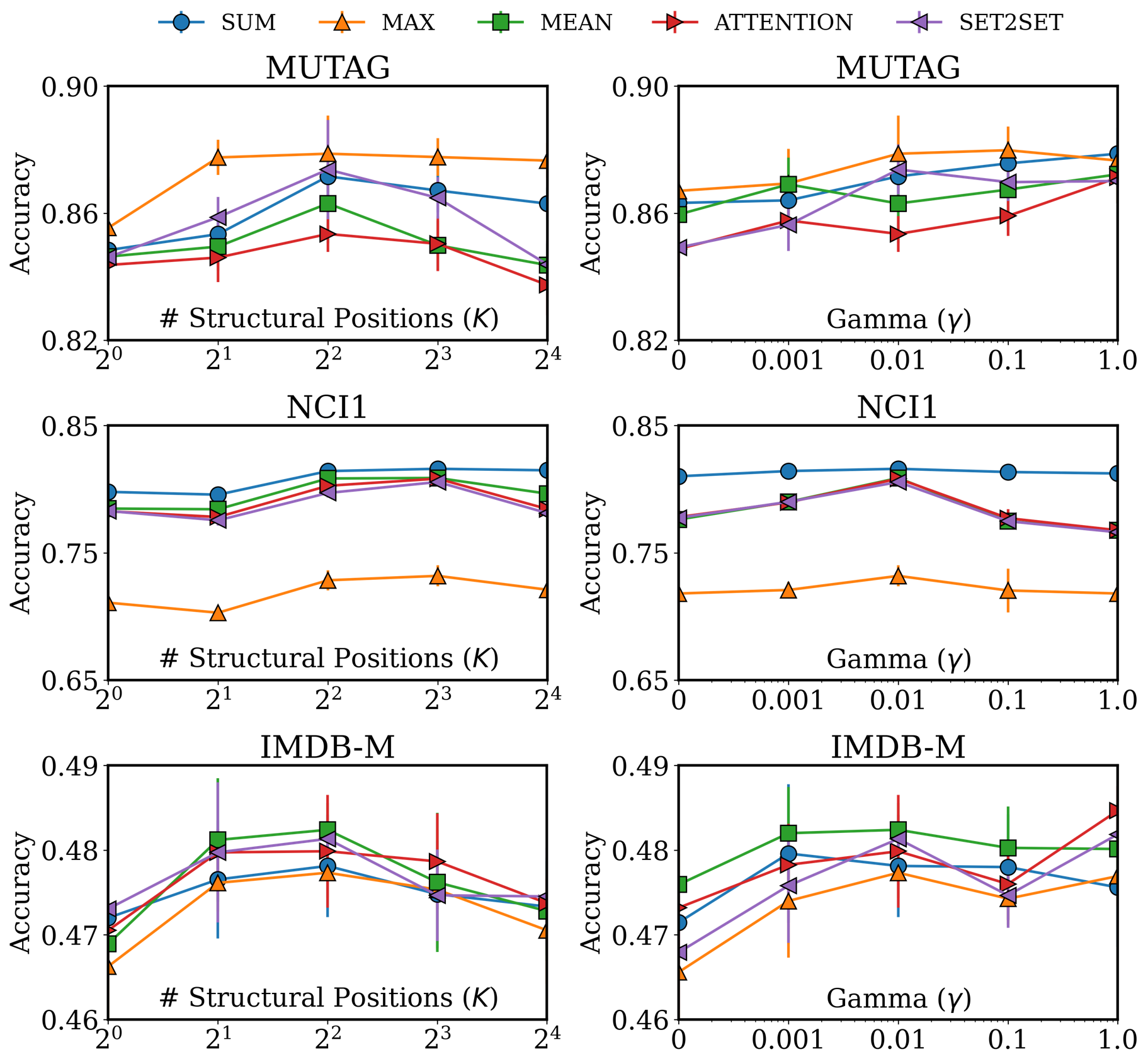}
	\caption{Performance change varying the number of structural positions (Left) and the smoothing parameter (Right). }
	\label{fig:sensitivity}
\end{figure}

\subsection{Sensitivity Analysis}
\label{subsec:sensitivity}
We investigate the performance changes of the base GNN classifier with respect to (i) the number of structural positions $K$ and (ii) the smoothing parameter $\gamma$ in the soft-SSA cost.
In terms of $K$, for each dataset, we observe a consistent trend across the aggregation functions (Figure~\ref{fig:sensitivity}, Left).
Specifically, the GNN classifier performs the best at $K=4$ for \mutag and \imdbm, and $K=8$ for \nci.
This is because all graph instances in the same dataset share global properties, such as the number of nodes and the combination of substructures, which are directly relevant to the optimal number of structural positions.
Thus, we tune the value of $K$ for each dataset, as described in Section~\ref{subsubsec:expset}.

On the other hand, the optimal $\gamma$ value varies depending on the aggregation functions as well as the target dataset (Figure~\ref{fig:sensitivity}, Right).
Thus we simply fix it by $\gamma=0.01$ rather than using its optimal value in our experiments, for a fair setting.
As an ablation analysis on $\gamma$, it is obvious that $\gamma=0$ leads to significantly lower performance than its appropriate value $\gamma>0$, which support our claim that the naive gradient of hard alignment cost (Equation~\eqref{eq:align}) cannot effectively optimize the structural prototypes for better alignment.

Finally, we report the training time of the state-of-the-art GNN classifiers on the \dnd and \nci datasets,\footnote{In Table~\ref{tbl:datastats}, each graph from the \dnd dataset has a large number of nodes, while the \nci dataset contains a large number of graph instances.} increasing $K$ from 2 to 16.
In Figure~\ref{fig:trainingtime}, the training time of the GNN classifiers with \proposed is not much different compared to the case of \gread.
This strongly indicates that the computation of the optimal alignment hardly affects the total training time, as discussed in Section~\ref{subsubsec:complexity}.
Moreover, the forward and backward computation of Equation \eqref{eq:softssa} (Lines 7 and 9 in Algorithm~\ref{alg:pseudocode}) does not incur much cost by Theorem~\ref{thm:softssa}, which facilitates the efficient optimization.

\begin{figure}[t]
	\centering
	\includegraphics[width=\linewidth]{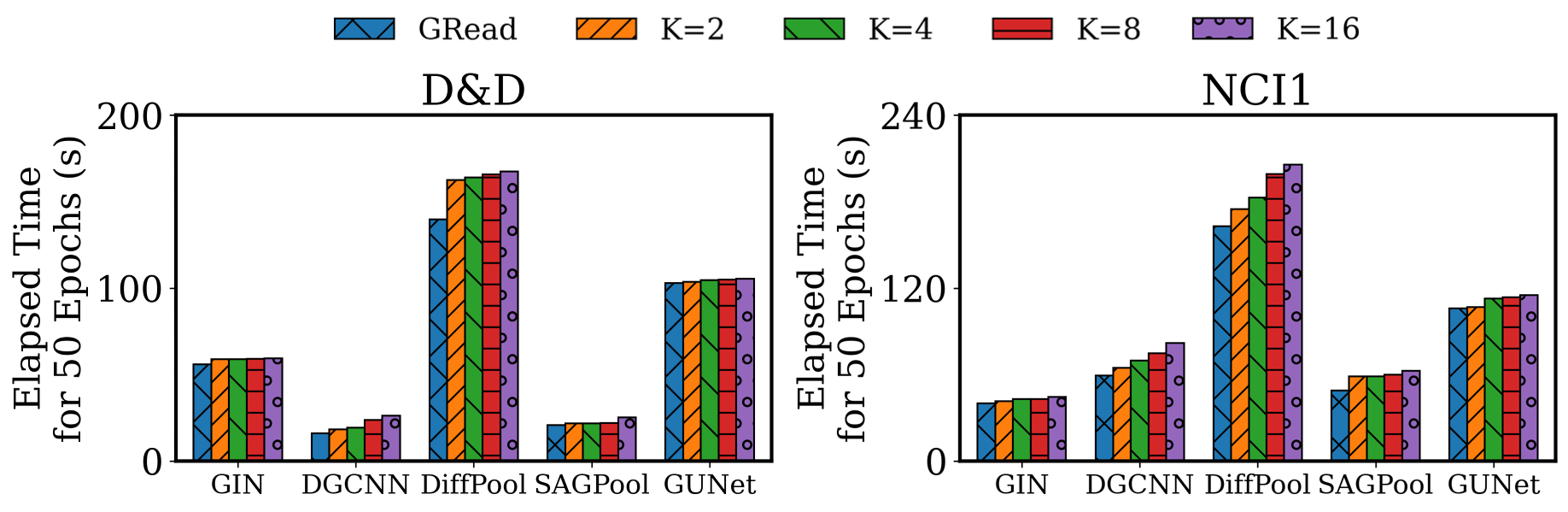}
	\caption{Training time of the GNN classifiers, increasing $K$.}
	\label{fig:trainingtime}
\end{figure}

\section{Conclusion}
\label{sec:conc}
This paper proposes a novel graph readout technique, named as \proposed, which outputs structured (or position-level) representations in order to explicitly leverage the global structural information for graph classification.
To this end, \proposed first identifies the structural position of the nodes by using the semantic alignment between the node representations and the structural prototypes, which are optimized to best summarize $K$ structural semantics observed in the training graphs.
Then, it produces position-level representations by aggregating the node representations in each position.
Our experiments support that \proposed consistently enhances the classification performance and interpretability of the GNN classifier while providing great compatibility with various aggregation functions, GNN architectures, and learning frameworks.

\smallsection{Acknowledgement}
This work was supported by the NRF grant funded by the MSIT (No. 2020R1A2B5B03097210, 2021R1C1C1009081), and the IITP grant funded by the MSIT (No. 2018-0-00584, 2019-0-01906).

\bibliographystyle{IEEEtran}
\bibliography{BIB/bibliography}

\end{document}